\documentclass{article}

\usepackage{microtype}
\usepackage{graphicx}
\usepackage{booktabs} 

\usepackage{hyperref}

\usepackage[accepted]{icml2020}

\icmltitlerunning{Representation Learning for Frequent Subgraph Mining}

\usepackage{booktabs}       
\usepackage{amsfonts}       
\usepackage{microtype}      
\usepackage{xcolor}
\usepackage{url}
\usepackage{verbatim} 
\usepackage{graphicx}
\usepackage{multirow}
\usepackage{algorithm}
\usepackage{xspace}
\usepackage{epsfig}
\usepackage{amsmath}
\usepackage{amsthm}
\usepackage{xr}
\usepackage{bbm}
\usepackage{bm}
\usepackage{subcaption}
\usepackage{enumitem}
\usepackage{hyperref}       
\usepackage{multicol}
\usepackage{tabularx}
\usepackage{perpage} 
\MakePerPage{footnote} 
\usepackage{wrapfig}

\newcommand{\andrew}[1]{{{\textcolor{teal}{[Andrew: #1]}}}}

\newcommand{\rex}[1]{{{\textcolor{magenta}{[Rex: #1]}}}}
\newcommand{\jure}[1]{{{\textcolor{red}{[Jure: #1]}}}}

\newcommand{\xhdr}[1]{{\noindent\bfseries #1}.}

\newcommand{\CITE}{{\textcolor{red}{[CITE]}}}

\newcommand{\cut}[1]{}

\newtheorem{proposition}{Proposition}
\newtheorem{definition}{Definition}

\newtheorem{problem}{Problem}

\newcommand{\cuthide}[1]{}

\newcommand{\name}{SPMiner\xspace}

\usepackage{amsmath,amsfonts,bm}

\def\eqref#1{equation~\ref{#1}}

\def\1{\bm{1}}

\DeclareMathAlphabet{\mathsfit}{\encodingdefault}{\sfdefault}{m}{sl}
\SetMathAlphabet{\mathsfit}{bold}{\encodingdefault}{\sfdefault}{bx}{n}

\DeclareMathOperator*{\argmax}{arg\,max}
\DeclareMathOperator*{\argmin}{arg\,min}

\newcommand{\norm}[1]{\left\lVert#1\right\rVert}

\begin{document}

\twocolumn[
\icmltitle{Representation Learning for Frequent Subgraph Mining}



\begin{icmlauthorlist}
\icmlauthor{Rex Ying}{yale}
\icmlauthor{Tianyu Fu}{tsinghua}
\icmlauthor{Andrew Wang}{stanford}
\icmlauthor{Jiaxuan You}{uiuc}
\icmlauthor{Yu Wang}{tsinghua}
\icmlauthor{Jure Leskovec}{stanford}
\end{icmlauthorlist}

\icmlaffiliation{yale}{Department of Computer Science, Yale University, New Haven, USA}
\icmlaffiliation{tsinghua}{Department of Electronic Engineering, Tsinghua University, Beijing, China}
\icmlaffiliation{stanford}{Department of Computer Science, Stanford University, Stanford, USA}
\icmlaffiliation{uiuc}{Department of Computer Science, University of Illinois at Urbana-Champaign, Urbana, USA}

\icmlcorrespondingauthor{Rex Ying}{rex.ying@yale.edu}

\icmlkeywords{Machine Learning, ICML}

\vskip 0.3in
]



\printAffiliationsAndNotice{}  

\begin{abstract}

Identifying frequent subgraphs, also called {\em network motifs}, is crucial in analyzing and predicting properties of real-world networks.
However, finding large commonly-occurring motifs remains a challenging problem not only due to its NP-hard subroutine of subgraph counting, but also the exponential growth of the number of possible subgraphs patterns.
Here we present {\em Subgraph Pattern Miner (\name)}, a novel neural approach for approximately finding frequent subgraphs in a large target graph. \name combines graph neural networks, order embedding space, and an efficient search strategy to identify network subgraph patterns that appear most frequently in the target graph.
\name first decomposes the target graph into many overlapping subgraphs and then encodes each subgraph into an order embedding space. \name then uses a monotonic walk in the order embedding space to identify frequent motifs. 
Compared to existing approaches and possible neural alternatives, \name is more accurate, faster, and more scalable.
For 5- and 6-node motifs, we show that \name can almost perfectly identify the most frequent motifs while being 100x faster than exact enumeration methods. 
In addition, \name can also reliably identify frequent 10-node motifs, which is well beyond the size limit of exact enumeration approaches. And last, we show that \name can find large up to 20 node motifs with 10-100x higher frequency than those found by current approximate methods.

\cut{
Identifying frequent subgraph structures have been crucial in analyzing and predicting properties of graphs.  
However, finding large commonly-occurring subgraphs remains an open problem due to its NP-hard subroutine of subgraph counting, and the combinatorial growth of the number of possible subgraphs with their size. 
Here we present GMiner, a novel neural approach to mining frequent subgraphs. GMiner integrates graph neural networks, order embedding space, and an efficient discrete search strategy to reliably identify subgraph patterns that appear most frequently in a target graph dataset. 
GMiner uses an expressive encoder to map graphs to order embeddings, and a pattern generator to identify frequent subgraphs using monotonic walks in the order embedding space.
Compared to existing state-of-the-art approaches and possible neural alternatives, GMiner is more accurate, faster, and more scalable.
For small subgraph sizes where exact groundtruth can be computed, we show that GMiner identifies subgraphs with frequency within 10\% of the groundtruth subgraph frequency. In addition, GMiner can reliably identify synthetically planted 10-node subgraphs, which is well beyond the size limit of present approaches.
We further test GMiner on real-world datasets, and show that it can find large-sized subgraphs that appear 100x more frequently than those subgraphs found by baselines.
\jure{what does this last sentence mean? clarify?}
}
\end{abstract}


\section{Introduction}
\label{sec:intro}
Finding frequently recurring subgraph patterns or {\em network motifs} in a graph dataset is important for understanding structural properties of complex networks \cite{benson2016higher}.
Frequent subgraph mining is a challenging but important task in network science: given a target graph or a collection of graphs, it aims to discover subgraphs or motifs that occur frequently in this data \cite{kuramochi2004efficient}.
In biology, subgraph counting is highly predictive for disease pathways, gene interaction and connectomes \cite{agrawal2018large}. In social science, subgraph patterns have been observed to be indicators of social balance and status \cite{leskovec2010signed}. In chemistry, common substructures are essential for predicting molecular properties \cite{cereto2015molecular}. 

However, frequent subgraph mining has extremely high computational complexity. A traditional approach to motif mining is to enumerate all possible motifs $Q$ of size up to $k$, usually $k\le 5$, and then count appearances of each $Q$ in a given graph \cite{hovcevar2014combinatorial}. This is problematic as the number of motifs $Q$ increases super exponentially with their size $k$, and counting the number of occurrences of a single motif $Q$ in the target graph $G$ is  NP-hard by itself~\cite{ribeiro2021survey}. 

More recently, neural approaches to learning combinatorially-hard graph problems, such as prediction of edit distance \cite{bai2019simgnn}, graph isomorphism \cite{fey2020deep,guo2018neural,li2019graph,xu2018how}, pairwise maximum common subgraphs \cite{bai2020neural} and substructure counting \cite{chen2020can, Liu2020NeuralSI}, have been explored.
Most recently, NeuroMatch~\cite{neuralsubgraphmatching} focuses on subgraph isomorphism testing, which attempts to predict whether a single query motif $Q$ is a subgraph of a large target graph $G$. While subgraph isomorphism is an important subroutine, frequent subgraph counting is much more challenging because it requires solving two intractable search problems: (i) counting the frequency of a given motif $Q$ in $G$; (ii) searching over all possible motifs to identify the ones with the highest frequency. Problem (i) is NP-hard; Problem (ii) is also hard because the number of possible graphs $Q$ increases super-exponentially with their size.

\begin{figure*}[t!]
    \centering
    \begin{tabular}[c]{cc}
    \begin{subfigure}[t]{0.5\textwidth}
        \centering
        \includegraphics[width=\textwidth]{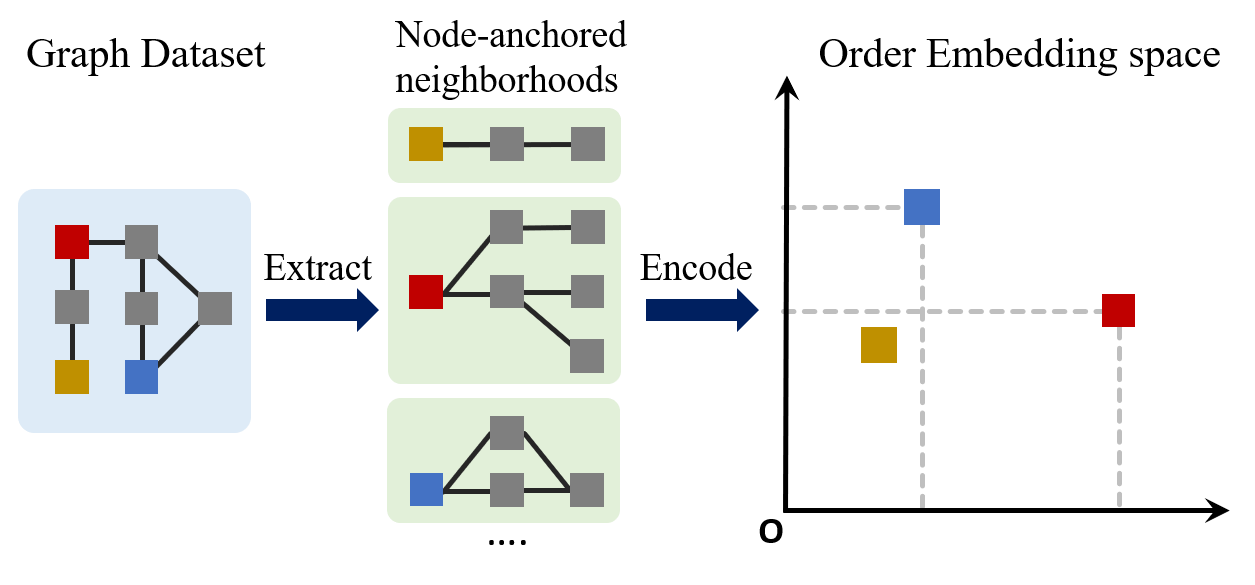} 
        \caption{\name Encoder}
    \end{subfigure} &
    \begin{subfigure}[t]{0.42\textwidth}
        \centering
        \includegraphics[width=\textwidth]{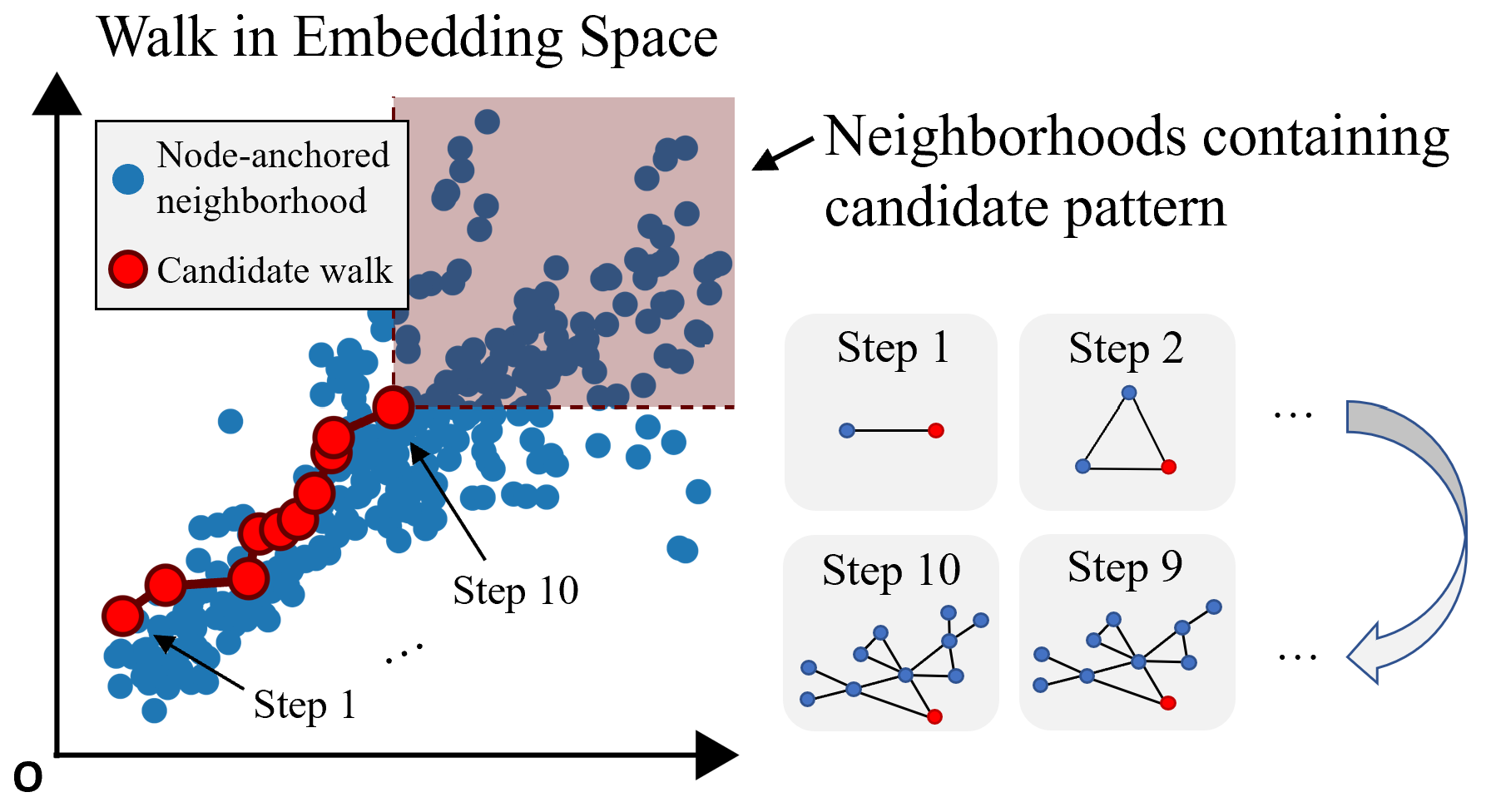}
        \caption{\name Decoder}
    \end{subfigure}%
    \end{tabular}
    \caption{\name encoder (a) and \name motif search procedure (b). (a) The \name decomposes a dataset into many node-anchored neighborhoods, and maps each neighborhood into a point in the embedding space such that order embedding property is preserved: if neighborhood $A$ is a subgraph of neighborhood $B$ then $A$ is embedded to the lower left of $B$. Here yellow node-anchored neighborhood is a  subgraph of both blue and red neighborhoods, so it is embedding to the lower left of both of them. %
    (b) \name then starts with an empty graph and iteratively adds nodes and edges to it to find frequent motifs. \name performs a monotonic walk in the order embedding space to identify a motif that is a subgraph of many neighborhoods. The walk in red represents growing of a frequent motif. Key insight here is that \name can quickly count the number of occurrences of a given motif by simply checking the number of neighborhoods (points) that are embedded to the top-right of it (denoted with a shaded region).
    }
    \label{fig:encoder_decoder}
\end{figure*}


Here we propose {\em Subgraph Pattern Miner (\name)}, a general framework using graph representation learning for identifying large frequent motifs in a large target graph. To the best of our knowledge, \name is the first neural approach to mining frequent subgraphs.

\name consists of two steps (Figure~\ref{fig:encoder_decoder}): {\bf (1) Embedding Candidate Subgraphs:}
\name decompose the input graph into overlapping node-anchored neighborhoods around each node. It then
uses an expressive graph neural network (GNN) to embed these neighborhoods to points in an \emph{order embedding} space. The order embedding space is trained to enforce the property that if one graph is a subgraph of another, then they are embedded to the ``lower left'' of each other (Figure \ref{fig:encoder_decoder} (a)). Hence the order embedding space captures the partial ordering of graphs under the subgraph relation. Importantly the GNN only needs to be trained once (using synthetic data) and then can be apply to {\em any} input graph.
{\bf (2)Motif Search Procedure:}
\name then directly reasons in the embedding space to identify frequent motifs of desired size $k$. \name searches for a \cut{discrete} $k$-step walk in the embedding space that stays to the lower left of as many neighborhoods (blue dots) as possible (Figure \ref{fig:encoder_decoder} (b)). The walk is performed by iteratively adding nodes and edges to the current motif candidate, and tracking its embedding\cut{ in the embedding space}. The key insight here is that \name can quickly count the number of occurrences of a given motif by simply checking the number of neighborhoods (points) that are embedded to the top-right of it in the embedding space.

\xhdr{Evaluation}
We carefully design an evaluation framework to evaluate performance of \name.
Current exact combinatorial frequent subgraph mining techniques only scale to motifs of up to 6 nodes. So, we first show that for 5- and 6-node motifs (where their exact count can be obtained), \name correctly identifies most of the top 10 most frequent motifs while being 100x faster than the exact enumeration. 
As present exact enumeration methods do not scale beyond motifs of size 6, we then generate synthetic graphs with planted frequent motifs of size 10, and again show that \name is able to robustly identify them. Last, we also compare \name to approximate methods for finding large motifs, and show that \name is able to identify large motifs that are 10-100x more frequent than those identified by current methods.


Overall, there are several benefits of our approach: robustness, accuracy and speed. In particular, \textbf{(a)} \name avoids expensive combinatorial graph matching and counting by mapping the problem into an embedding space; \textbf{(b)} \name allows for a neural model to estimate frequency of any motif $Q$ directly in the order embedding space; \textbf{(c)} Training only needs to be performed once on a proposed large synthetic dataset, and then \name can be applied to any graph dataset; \textbf{(d)} The embedding space can be efficiently navigated in order to identify large motifs with high frequency.

\section{Related Work}

Existing approaches on subgraph mining include two lines of work: subgraph counting and frequent subgraph mining.

\xhdr{Subgraph counting} There have been works in counting a query motif in a target graph. Hand-crafted schemes have been successful for small motifs (up to 5 nodes)~\cite{hovcevar2014combinatorial, jha2015path}, and approaches like statistical sampling~\cite{kashtan2004efficient}, random walk~\cite{chen2016general}, and redundancy elimination~\cite{shi2020graphpi, mawhirter2021graphzero} have been applied. 
However, these methods are unscalable for the subgraph mining task, since they would require intractable enumeration of all possible motifs of a given size.

\xhdr{Non-ML approaches to frequent subgraph mining}
\name aims to solve the problem of frequent subgraph mining, which is much more challenging than subgraph counting as it involves searching for the most frequent subgraphs.
There exist exact and approximate methods for frequent subgraph mining that involve searching the possible subgraphs by pruning or compression.
Exact methods often use order restrictions to reduce the search space\cite{yan2002gspan, kuramochi2004efficient, nijssen2005gaston} or use decomposition for small subgraphs\cite{pinar2017escape}.
Heuristic and sampling methods offer faster execution time while no longer guaranteeing that all frequent subgraphs will be found: greedy beam search~\cite{ketkar2005subdue}, pattern contraction~\cite{matsuda2000extension}, subgraph sampling~\cite{wernicke2006efficient} and color coding~\cite{bressan2019motivo, bressan2021faster} have been used, which reduces the size of the subgraph search space to scale with target graph size. 
However, these approaches scale poorly with increasing query graph size due to combinatorial growth of the sample space; and the precision degrades significantly for larger subgraphs.

\xhdr{Neural approaches}
Recently there are approaches that use graph neural networks (GNNs) to predict relations  
between graphs: models that learn to predict graph edit distance~\cite{bai2019simgnn,li2019graph}, graph isomorphism~\cite{fey2020deep,guo2018neural,xu2018how,xu2019cross}, and substructure counting~\cite{chen2020can, Liu2020NeuralSI} have been developed. Neural models have also been proposed to find maximum common subgraphs~\cite{bai2020neural} and generate prediction explanations \cite{ying2019gnnexplainer}, but only in the context of a single prediction or graph pair rather than across an entire dataset. 
Our work is inspired by the recent approach to neural subgraph matching~\cite{neuralsubgraphmatching}.
However, here we are solving a different and a much harder problem: our goal is not only to identify the frequency of a given motif $Q$ in graph $G$ but also identify motifs $Q$ that have high frequency.





\section{Proposed Method}
We first introduce subgraph mining problem and its subroutine of subgraph isomorphism, then define our objective of finding frequent motifs. 
We then introduce \name, which has two components: an encoder that maps graphs into embeddings to capture subgraph relation, and a motif search procedure that identifies motifs that appear most frequently in the graph dataset.


\subsection{Problem Setup}

Let $G_T = (V_T, E_T)$ be a large \emph{target graph} where we aim to identify common motifs, with vertices $V_T$ and edges $E_T$. 
For notational simplicity, we consider the case of a single target graph. A dataset of multiple graphs could be considered as a large target graph with multiple disconnected components.

\begin{figure}[t]
    \centering
    \includegraphics[width=0.3\textwidth]{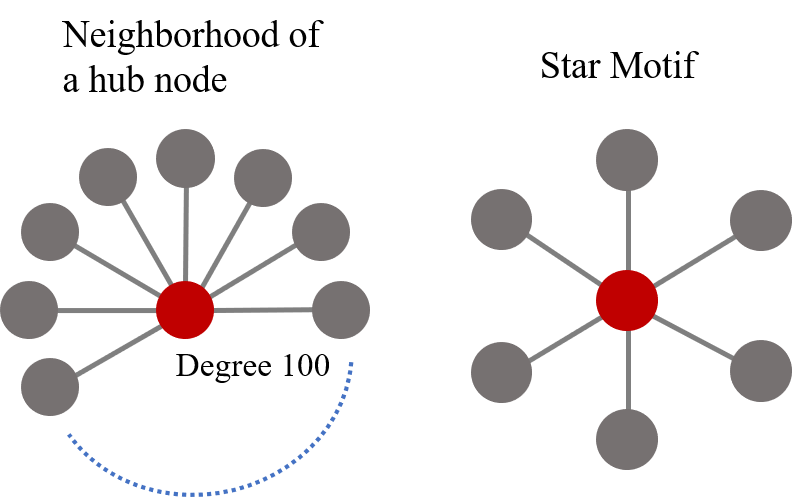} 
    \caption{Distinction between Node-anchored and Graph-level subgraph frequency. Consider a hub node with degree 100 (left). We aim to determine frequency of the star motif (right). Definition 1 with center anchor results in a count of $1$. In contrast, Definition 2 counts in ${100\choose 6}$ motif occurrences.}
    \label{fig:motif_freq_example}
\end{figure}

Analogously, let $G_Q = (V_Q, E_Q)$ be a \emph{query motif}. 
The \textit{subgraph isomorphism problem} is to determine whether an isomorphic copy of the query graph $G_Q$ appears as a subgraph of the target graph $G_T$. Formally, $G_Q$ is subgraph isomorphic to $G_T$ if there exists 
%
%
%
an injection $f: V_Q \mapsto V_T$ such that $(f(v), f(u)) \in E_T \iff (v, u) \in E_Q$ (all corresponding edges match). 
The function $f$ is a \textit{subgraph isomorphism mapping}. Following most previous motif mining literature, we focus on mining \textit{node-induced} subgraphs following many previous works~\cite{akihiro2000apriori,kuramochi2004efficient,yan2002gspan}, where subgraph edges are induced by the subset of nodes. However our method can also be applied to the variation of mining edge-induced subgraphs\footnote{The only change is to adjust the training set to sample edge-induced subgraph pairs instead of node-induced subgraph pairs in Section~\ref{sec:encoder}.}.

The problem of \textit{frequent subgraph mining} is to identify subgraph patterns (i.e., motifs) that appear most frequently in a given dataset $G_T$. We focus on the case of finding \textit{node-anchored} motifs $(G_Q, v)$ \cite{benson2016higher,hovcevar2014combinatorial}.
Throughout, we say that $G_Q$ anchored at $v \in V_Q$ is a subgraph of $G_T$ anchored at $u \in V_T$ if there exists a subgraph isomorphism $f: V_Q \mapsto V_T$ satisfying $f(v) = u$.
Our goal is to find most frequent motifs $G_Q$ with associated \textit{anchor nodes} $v \in V_Q$.
We define the frequency of a node-anchored subgraph pattern as:

\begin{definition}{\textbf{Node-anchored Subgraph Frequency.}}
\label{def:motif_freq}
Let $(G_Q, v)$ be a node-anchored subgraph pattern. The frequency of motif $G_Q$ in the graph dataset $G_T$, relative to anchor node $v$, is the number of nodes $u$ in $G_T$ for which there exists a subgraph isomorphism $f: V_Q \mapsto V_T$ such that $f(v) = u$~\footnote{Note that standard subgraph isomorphism can be reduced to anchored subgraph isomorphism by adding an auxiliary anchor node to both the query and target graph that is connected to all existing nodes.}.
\end{definition}

We remark that Definition 1 is the Minimum Image Support frequency \cite{bringmann2008frequent} or graph transaction based frequency\cite{Jiang2012ASO} from existing literature, extended to the node-anchored setting. 
Although we focus on the node-anchored frequency definition in designing methods, in experiments we additionally evaluate performance with the alternative graph-level frequency definition, and show that \name works under both definitions.
\begin{definition}{\textbf{Graph-level Subgraph Frequency.}}
\label{def:motif_freq_graph}
Let $G_Q$ be a subgraph pattern. The graph-level frequency of $G_Q$ in $G_T$ is the number of unique subsets of nodes $S \subset V_T$ for which there exists a subgraph isomorphism $f: V_Q \mapsto V_T$ whose image is $S$.
\end{definition}


Compared with graph-level frequency, the node-anchored definition is 1) robust to outliers, 2) provides more holistic understanding of subgraphs, and 3) satisfies the Downward Closure Property (DCP). 
Fig. \ref{fig:motif_freq_example} shows that a highly self-symmetric pattern. It occurs combinatorially many times in the neighborhood under Definition 2, which gives an outlier count that overshadows the counts of other important neighborhoods in the target graph. 
Choosing the center and peripheral nodes as the anchor node under Definition 1 will result in different counts (1 and 100). Different counts based on different anchor nodes of the same subgraph holistically describe the frequency of appearances of nodes of different roles in the pattern.
Furthermore, Definition 1 has an important Downward Closure Property (DCP), which is valued by previous works\cite{yan2002gspan, nijssen2005gaston}. DCP bounds the count of large query by that of its subgraphs. So the number of center-anchored 6-star is no larger than center-anchored 5-star, 4-star, etc. 

The goal of \name is to identify subgraphs of maximum frequency:


\begin{problem}{\textbf{Goal of \name.}}
Given a target $G_T$, a motif size parameter $k$ and desired number of results $r$, the goal of \name is to identify, among all possible graphs on $k$ nodes, the $r$ graphs (i.e, motifs) with the highest frequency in $G_T$. 
\end{problem}

Note that the anchors do \emph{not} need to be specified by the user. Instead, the frequent motif output by the mining algorithm contain an anchor node (see Figure~\ref{fig:motif_freq_example} for example). Depending on downstream applications, the anchor information can be ignored. If multiple frequent motifs are needed, graph isomorphism test (a much easier task) can be used to de-duplicate predicted top-K frequent motif patterns that have different anchors but are isomorphic.


\subsection{\name Encoder $\phi$: Embedding Candidate Subgraphs}
\label{sec:encoder}

We first provide a high-level overview of subgraph encoding, consisting of two steps:
(1) Given $G_T$, we decompose it by extracting $k$-hop neighborhoods $G_v$ (Definition~\ref{def:neighborhood}) anchored at each node $v$.
In \name the neighborhoods are extracted via breadth-first search procedure.
(2) The encoder $\phi$ is a graph neural network (GNN) to map the neighborhoods $G_v$ into an order embedding space (Figure~\ref{fig:encoder_decoder}).
\begin{definition}{\textbf{Neighborhood.}}
The $k$-hop neighborhood anchored at node $v$ contains all nodes that have shortest path length at most $k$ to node $v$.
\label{def:neighborhood}
\end{definition}

\cut{
\jure{we have to give a high-level overview of the encoding and reasoning behind it:\\
1) take big $G_T$\\
2) decompose it into node-anchored subgraphs (of size at least $k$)\\
3) embed subgraphs by using a GNN centered at the node anchor (not this is different than how graphs are generally embedded (using an average of individual node embeddings)\\
4) GNN is trained to capture subgraph relations (order embedding)\\
This is the core of the method. Everything else are details.
\\
\\
It is important to establish duality or correspondence between the node
}}


\xhdr{Mapping node-anchored neighborhood to embedding space}
Subgraph Frequency Definition~\ref{def:motif_freq} uses the concept of node-anchored subgraph isomorphism.
We use a categorical node feature to represent whether a node is an anchor $v$ of a neighborhood graph $G_v$, embed it by computing node embeddings of $G_v$ through a GNN and aggregate into the neighborhood embedding by sum pooling. 

\xhdr{Order embedding space}
Order embedding \cite{vendrov2016order} is a representation learning technique that uses the geometric relations of embeddings to model a partial ordering structure. Order embeddings are a natural way to model subgraph relations because subgraph isomorphism induces a partial ordering on the set of all graphs via 
its properties of \emph{transitivity} and \emph{antisymmetry}. 

\begin{figure}[tb]
    \includegraphics[width=0.45\textwidth]{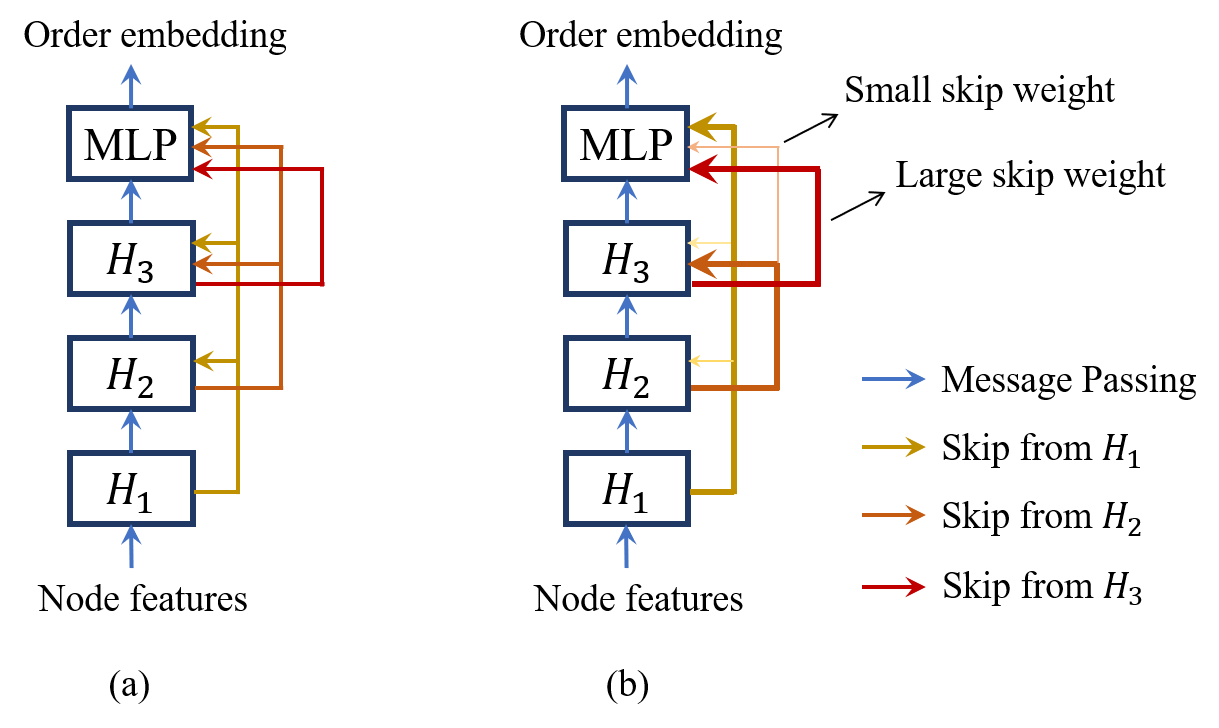} 
    \vspace{-2mm}
    \caption{\name\ Learnable skip layer. (a) Initially all skip connections are assigned equal weights. (b) After training the learnable skip GNN, the model learns the best skip connection configurations that encode subgraph relations. The architecture only requires $O(L^2)$ additional paramters. }
    \label{fig:skip}
\end{figure}

Formally, we define a partial order $\preceq$ on the set of all graphs $\mathcal{G}$. Let $A, B \in \mathcal{G}$, and denote $A \preceq B$ if graph $A$ is isomorphic to a subgraph of $B$. We then assume we are given an embedding function $\phi: \mathcal{G} \mapsto \mathbb{R}^n$ that maps graphs to vectors, enforcing the \emph{order embedding constraint} that $A \preceq B$ if and only if $\phi(A) \leq \phi(B)$ elementwise. In other words, embedding $\phi(A)$ is to the ``lower left'' of embedding $\phi(B)$. 

The key property of the space is that \name can quickly count the number of occurrences of a given motif by simply counting the number of points in the embedding space (i.e., node-anchored neighborhoods) that are to the top-right of the motif's embedding.

\xhdr{\name Graph Neural Network}
\name  uses a Graph Neural Network (GNN) to learn an embedding function $\phi$, which maps node-anchored neighborhoods into points in the embedding space such that the subgraph property is preserved. Importantly, \name  GNN is only trained once and then can be applied to any target graph from any domain. This is due to the fact, that GNN needs to learn to map different subgraphs to different points in the embedding space and once it is trained it can be applied. This means that an application of \name to a new graph does not require any training.

To train the \name GNN we need to define a loss function and then optimize parameters of the GNN to minimize the loss. The \emph{order embedding penalty} between two graphs $A$ and $B$ (where $A$ is a subgraph of $B$) is defined as:
\begin{equation}
\label{eq:order_penalty}
    E(A, B) = \norm{\max(0, \phi(A) - \phi(B))}^2.
\end{equation}
We call $E(A, B)$ the \textit{penalty}. To enforce the order embedding constraint when training the \name GNN, we use this penalty in the following max-margin loss:
\begin{equation}
\label{eq:order_loss}
    \sum_{(A, B) \in P} E(A, B) + \sum_{(A', B') \in N} \max(0, \alpha - E(A', B'))
\end{equation}
Here, $P$ is the set of positive examples (pairs $A, B$ where $A$ is a subgraph of $B$) and $N$ is the set of negative examples (pairs that do not satisfy the subgraph relation); $\alpha$ is the margin hyperparameter which controls the separation between penalties of positive and negative examples.

Observe that at the query time, given precomputed embeddings of $A, B$, we can use $E(A, B)$ to quickly determine if $A$ is subgraph isomorphic to $B$, simply by checking if $E(A, B)$ is below a learned threshold. This is important as it allows us to test whether $A$ is a subgraph of $B$ in time linear in embedding dimension, independent of graph sizes.


\xhdr{\name GNN Architecture}
It is essential for our GNN to be expressive in capturing neighborhood structures~\cite{xu2018how}. We achieve this with a GNN of large depth. However, increasing depth can potentially degrade GNN model performance due to oversmoothing.
We propose a new approach of \emph{learnable skip layer}, based on the dense skip layers. 
Different from previous GNN skip layers \cite{hamilton2017inductive,li2019deepgcns}, we use a fully connected skip layer analogous to DenseNet~\cite{huang2017densely}, and additionally assign a learnable scalar weight $w_{i,l}$ to each skip connection from layer $i$ to $l$.
Let $1 \le l \le L$ be the layer number, and $H_l'$ be the embedding matrix at $l$-th layer after message passing.
At each layer $l$, the node embedding matrix $H_l$ is computed by:

\begin{equation}
\label{eq:skip}
    H_l = \mathrm{Concat} \left(\sum_{i=1}^{l-1} w_{i,l} H_i, H_l' \right), \ \forall l=1, 2, \ldots, L
\end{equation}

When learning embeddings for subgraph isomorphism tasks, we can view representations at layer $i$ as describing the graph structural information for the $i$-th hop neighborhood~\cite{hamilton2017inductive}. Learnable skip allows every layer of the model to easily access structural features of different sized neighborhoods. This ensures that only useful skip connections across layers are retained, and different sized subgraph components (at different layers) can be simultaneously considered.

\xhdr{Training the \name GNN}
To train the \name GNN, we generate a set of positive and negative training instances and then minimize the order embedding loss in Equation \ref{eq:order_loss}. 
In particular, we first sample a random subgraph anchored at a random node $v$ as a target neighborhood $G$. To generate a positive example, we sample a smaller random subgraph from the target neighborhood $G_v$. Negative examples are generated by random sampling of a different subgraph. 
This sampling setup allows us to circumvent exact computation of subgraph isomorphism or subgraph frequency, which would make training intractable.

\cut{
\xhdr{Subgraph Frequency}
In order to use the GNN for subgraph mining with Frequency Definition \ref{def:motif_freq}, we decompose $G_T$ into many overlapping node-anchored neighborhoods. For a large sample of nodes $v \in V_T$, we sample a neighborhood $G_v$ centered at $v$.
We then use the \name encoder GNN $\phi$ to embed the anchored graph $G_v$. 
Given an query motif $(G_Q, u)$, we then predict whether the query is a subgraph of each of the neighborhoods $G_v$ in $G_T$.  Both a hard and a soft frequency estimation can be obtained. A hard estimation counts the number of neighborhoods $G_v$ for which subgraph prediction is positive. However, a soft prediction based on the penalty scores $E$ proves to be more effective when generating common patterns, which we will explain in the next section.}

\subsection{\name Decoder: Motif Search Procedure}

\xhdr{Motif search procedure by walking in order embedding space}
Given the encoder GNN $\phi$ that maps node-anchored subgraphs to order embeddings, the goal of the search procedure is to identify node-anchored motifs that appear most frequently in the given graph dataset. \name uses the approach of generating frequent motifs by iterative addition of nodes. Proposition \ref{prop:walk} shows that the order embedding provides a well-behaved space that makes the search process efficient and effective:

\begin{proposition}
\label{prop:walk}
Given an order embedding encoder GNN $\phi$, let a graph generation procedure be $\{G_0, G_1, \ldots, G_{k-1}\}$,
where at any step $i$, $G_i$ is generated by adding 1 node to $G_{i-1}$.
Then the sequence $\{\phi(G_0), \phi(G_1), \ldots, \phi(G_{k-1})\}$ is a monotonic walk in the order embedding space, i.e. $\phi(G_0) \leq \ldots \leq \phi(G_{k-1})$ elementwise.
\end{proposition}
See Figure \ref{fig:encoder_decoder}(b) for an example of monotonic walk in the order embedding space in 2D.

In particular, we observe that node-anchored frequency monotonically decreases with each step of the monotonic walk:
\begin{proposition}
Given a motif $G_A$ with embedding $x_A$ and motif $G_B$ with embedding $x_B$, $x_A \leq x_B$ elementwise implies $\textrm{Freq}(G_A) \leq \textrm{Freq}(G_B)$ where Freq denotes frequency under Definition 1.
\end{proposition}

This proposition is an immediate consequence of the fact that for all node-anchored graphs $G$ that $G_B$ is a subgraph of, particularly the target graph $G_T$ anchored at each of its nodes, $G_A$ is also a subgraph of $G$. This property, known as the anti-monotonic property \cite{elseidy2014grami}, demonstrates that the monotonic walk enforces important structure on the search space. In particular, the frequency of a motif at any point in a monotonic walk is an upper bound on the frequency of all motifs at subsequent points on the walk, serving as an approximation for their frequencies. Thus, this quantity can be useful for guiding search in the order embedding space, a property that we leverage in the next section.

\xhdr{Hard frequency objective}
Assuming a perfect order embedding, the {\em frequent motif objective} is then translated to finding a walk with destination embedding such that the number of neighborhoods $G_v$ in dataset $G_T$ satisfying $\phi(G_{k-1}) \preceq \phi(G_v)$ is maximized. 

\xhdr{Soft frequency objective}
Recall that the \name GNN is trained with a max margin loss on the penalty $E$:
\begin{equation}
E(A, B) = \norm{\max(0, \phi(A) - \phi(B))}^2.
\end{equation}
The penalty can be interpreted as a measure of model confidence: the smaller the penalty is, the more confident is the model about $A$ being a subgraph of $B$.
To take model confidence into account, we define a continuous objective: 
find graph $G_{k-1}$ of given size $k$ that minimizes the total penalty $m(G_{k-1})$. Under Definition~\ref{def:motif_freq}, the predicted frequent subgraph $G_{\mathrm{freq}}$ is then:
\begin{align}
\label{eq:walk}
\begin{split}
    &\phi(G_{\mathrm{freq}}) = \argmin_{G\in \mathcal{G}} m(G), \\
    & \mathrm{where} \ m(G) = \sum_{N \in \mathcal{N}} \norm{\max(0, \phi(G) - \phi(N)}^2.
\end{split}
\end{align}
$\mathcal{G}$ is the set of all graphs of size $k$;
$\mathcal{N}$ is the set of all neighborhood graphs in $G_T$ for all nodes $v \in G$ (See Definition~\ref{def:neighborhood}).
The total penalty $m(G)$ is a \emph{soft estimation} of the number of neighborhoods containing the anchored subgraph $G$. 

\xhdr{Motif search}
Directly finding the frequent motif is hard since the number of possible motifs is exponential to the size, so we design a special search procedure that \textbf{iteratively} grows the motif.
In order to find a frequent motif of size $k$, \name randomly samples a \emph{seed node} from the dataset $G_T$, referred to as the trivial seed graph $G_0$ with size $1$.
Starting from the seed graph $G_0$, we iteratively generate the next graph by adding an adjacent node in $G_T$ (and its corresponding edges). Figure \ref{fig:encoder_decoder}(b) shows an example search process and the corresponding walk in the embedding space.
Proposition \ref{prop:walk} guarantees that the corresponding embedding increases monotonically as more nodes are added. Throughout the walk/generation, we make use of both $G_{i}$ and its embedding $\phi_{i}, \phi_{i} = \phi(G_{i})$. In practice, to attain a robust estimate, we sample several seed nodes, run several walks, and select the resulting motifs of size $k$ that we encountered the most times.

\xhdr{Search strategies: Greedy strategy, beam search and Monte Carlo Tree Search (MTCS)}
%
Our motif search procedure is general and different strategies can be implemented to navigate the space of motifs. 
We propose a greedy strategy to improve scalability. At every step, \name adds the node such that the total penalty $m$ in Eq. \ref{eq:walk} is minimized.
Let $G'$ be chosen by adding 1 adjacent node to $G_i$ in $G_T$.
The greedy approximation simplifies Equation \ref{eq:walk} into a step-wise minimization:
$\forall i = 0, 1, \ldots, k-1,$
\begin{equation}
\label{eq:greedy}
    G_{\mathrm{freq}} = G_{k-1}, \quad G_{i+1} = \argmin_{G'} m(\phi(G')).
\end{equation}
A beam search strategy strikes a balance between the greedy and exhaustive strategies, where instead of greedily adding the next node resulting in the smallest penalty, we explore a fixed number of options with small total penalty scores to add nodes at each generation step.

The \name framework also naturally lends itself to Monte Carlo Tree Search (MCTS) strategies \cite{coulom2006efficient} with neural value function. 
\name MCTS runs multiple walks starting from multiple seed nodes and maintains a visit count $n(G_i)$ and a total value $f(G_i)$ for each step $G_i$ in each walk explored so far. The visit counts $n(G_i)$ refer to the number of times where $G_i$ is visited. Graphs that share an embedding also share a visit count, except for seed nodes;
this setup allows \name to revisit promising seeds and explore new ones. We design $f(G_i) = \sum_{G'} (1-\log(\frac{m(G')}{|\mathcal{N}|} + 1))$, where $G'$ ranges over all size-$k$ graphs reached from any walk that visited $G_i$.
We design the following objective based on the upper confidence bound criterion for trees (UCT) \cite{kocsis2006bandit} to replace the greedy approach (Eq. \ref{eq:greedy}), with an exploration constant $c$:
$\forall i = 0, 1, \ldots, k-1,$
\begin{equation}
    \quad G_{i+1} = \argmax_{G'} \left( \frac{f(G')}{n(G')} + c \sqrt{\frac{\log n({G_i)}}{n(G')}}\right).
\end{equation}
The use of $f$ rather than the total penalty $m(G')$ ensures that the first term has the same numerical scale as the second term. 
In the end, the motifs with highest visit count are selected.

\subsection{Runtime and Memory Analysis}
The memory cost of \name is $O((N+K)d + Kn^2)$, and its runtime is $O(Nb^2 + Kn(n^2+Nnd))$, where $N$ is the number of neighborhoods, $d$ is embedding dimension, $K$ is number of decoder iterations, $n$ is the desired motif size and $b$ is neighborhood size. The polynomial runtime and memory usage enable efficient mining of large motifs.

\subsection{\name Expressive Power}
\label{sec:analysis}
\xhdr{Expressiveness of \name\ encoder GNN}
To show the expressiveness of \name GNN encoder with order embedding, we first show the existence of a perfect order embedding, and then demonstrate that the \name\ encoder architecture can capture information in the perfect order embedding.

In addition to satisfying the properties of subgraph relations, the following proposition demonstrates the full expressive power of the order embedding space in predicting subgraph relations.

\begin{proposition}
Given a graph dataset of size $n$, we can find a perfect order embedding of dimension $D$, where $D$ is the number of non-isomorphic node-anchored graphs of size no greater than $n$. The perfect order embeddings satisfy $z(G_Q) \preceq z(G)$ if and only if $G_Q$ is a subgraph of $G$, where $z(G)$ is the order embedding of $G$.
\end{proposition}

This can be proven by constructing an order embedding and enumerating all possible non-isomorphic node-anchored graphs of size no greater than $n$, and placing the count of the $i$-th graph into the $i$-th dimension of the order embedding. The proposition indicates that order embeddings can achieve perfect expressiveness. Although $D$ can be large in theory, in practice, the neural model can learn a more compressed embedding space, and we find $D=64$ is sufficient to achieve good performance.


\cut{
\xhdr{Motif mining objective analysis}
We further analyze our objective of maximizing the subgraph frequency as defined in Definition \ref{def:motif_freq}.
In comparison, the classic motif frequency definition\CITE counts the total number of motif appearances in the given graph dataset. 
}


\subsection{Synthetic Graph Pretraining}

\begin{figure}[tb]
    \centering
    \includegraphics[width=0.23\textwidth]{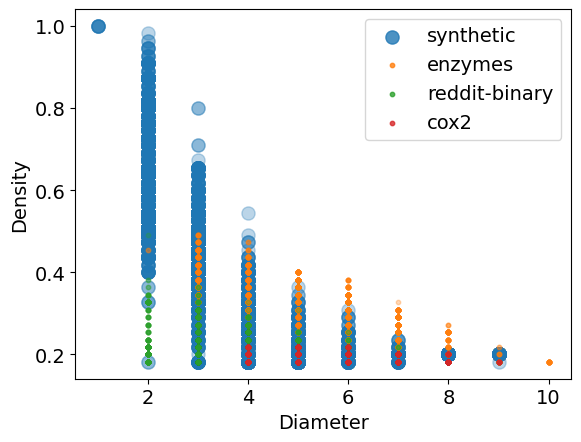}
    \includegraphics[width=0.23\textwidth]{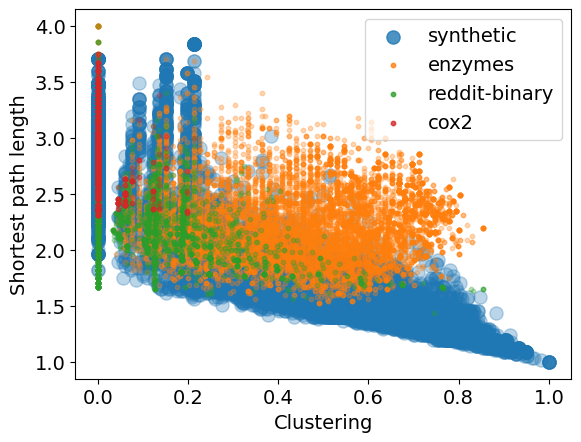}
    \vspace{-2mm}
    \caption{Graph statistics of the synthetic and real-world graph datasets. Each point represents the statistics of one graph; the color of a point represents the dataset that the graph belongs to.}
    \label{fig:syn-real-stats}
\end{figure}

The \name GNN needs to be trained only once and can be applied to any input graph. In fact, we train the \name GNN on a set of synthetic graphs.
%
We generate millions of synthetic training datasets to learn a general order embedding space agnostic to the dataset domain. This pretrained GNN model can then be immediately applied to real-world datasets not seen in training without further fine-tuning.

In contrast to the previous works, we use a combination of graph generators, such as Erdos-Renyi, Power Law Cluster and others, to ensure that the model can be trained on a diverse set of graphs in terms of graph properties (parameter specification in Appendix \ref{appendix: synthetic dataset}). 

Using this graph generator, we create a balanced dataset of graph pairs $(A, B)$ in which a pair is positive if $B$ is a subgraph of $A$. To create positive pairs, we sample a graph $A$ from the generator and sample a subgraph $B$ of $A$ using the sampling procedure of \textsc{MFinder}~\cite{kashtan2004efficient}. To create negative pairs, we sample a graph $A$ from the generator. With 50\% probability, we sample a subgraph $B$ of $A$, then randomly add up to 5 edges to $B$ so that it is unlikely to be a subgraph of $A$. Otherwise, we randomly sample another graph from the generator, and the sample is unlikely to be a subgraph of $A$. We sample $A$ of size uniform from 6 to 29 and $B$ of size uniform from 5 to $|A|-1$. Empirically we observe that the model is able to continue improving in performance even after seeing millions of training pairs, hence a large synthetic dataset provides an important performance advantage.

\xhdr{Synthetic dataset statistics}
We demonstrate that our synthetic data generation scheme is capable of generating graphs with high variety. Figure \ref{fig:syn-real-stats} shows the statistics of the synthetic graphs, compared to real-world datasets. In terms of graph statistics, including density, diameter, average shortest path length, and average clustering coefficient, the synthetic dataset (large blue dots) covers most of the real-world datasets, including those in the domains of chemistry (\textsc{COX2}), biology (\textsc{ENZYMES}) and social networks (\textsc{Reddit-Binary}) \cite{Fey/Lenssen/2019}.

The high coverage of statistics suggests that the synthetic dataset is an application-agnostic dataset that allows \name to learn a highly generalizable order embedding model, and immediately apply it to analyzing new real-world datasets without further training.
\begin{figure*}[ht]
    \centering
    \begin{tabular}{p{0.35\textwidth} p{0.3\textwidth} p{0.3\textwidth}}
    \vspace{-1pt} 
    \setlength\tabcolsep{4pt}
    \begin{tabular}{c|c|c|c}
        \multicolumn{4}{c}{\textbf{Hit rate (size-6 graphs)}} \\
        \hline
        Rank & SPMiner & MFinder & Rand-ESU  \\
        \hline
        10 & 0.90 & 0.20 & 0.30 \\
        20 & 0.85 & 0.20 & 0.45 \\
        30 & 0.77 & 0.37 & 0.57 \\
        40 & 0.75 & 0.40 & 0.60 \\
        50 & 0.70 & 0.50 & 0.56 \\
        \hline
    \end{tabular}
    &
    \vspace{-1pt} 
    \includegraphics[width=0.3\textwidth]{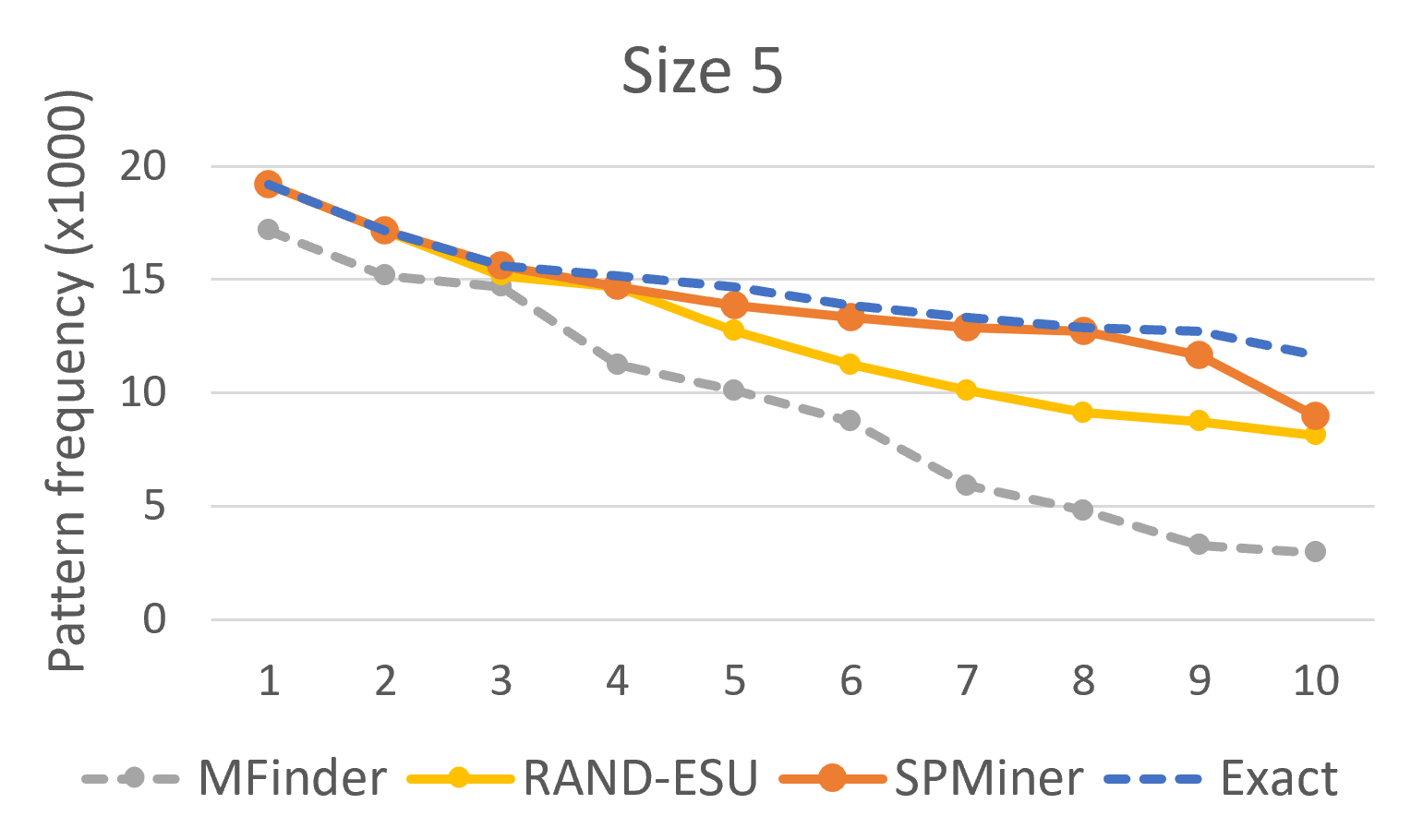} & 
    \vspace{-1pt} 
    \includegraphics[width=0.3\textwidth]{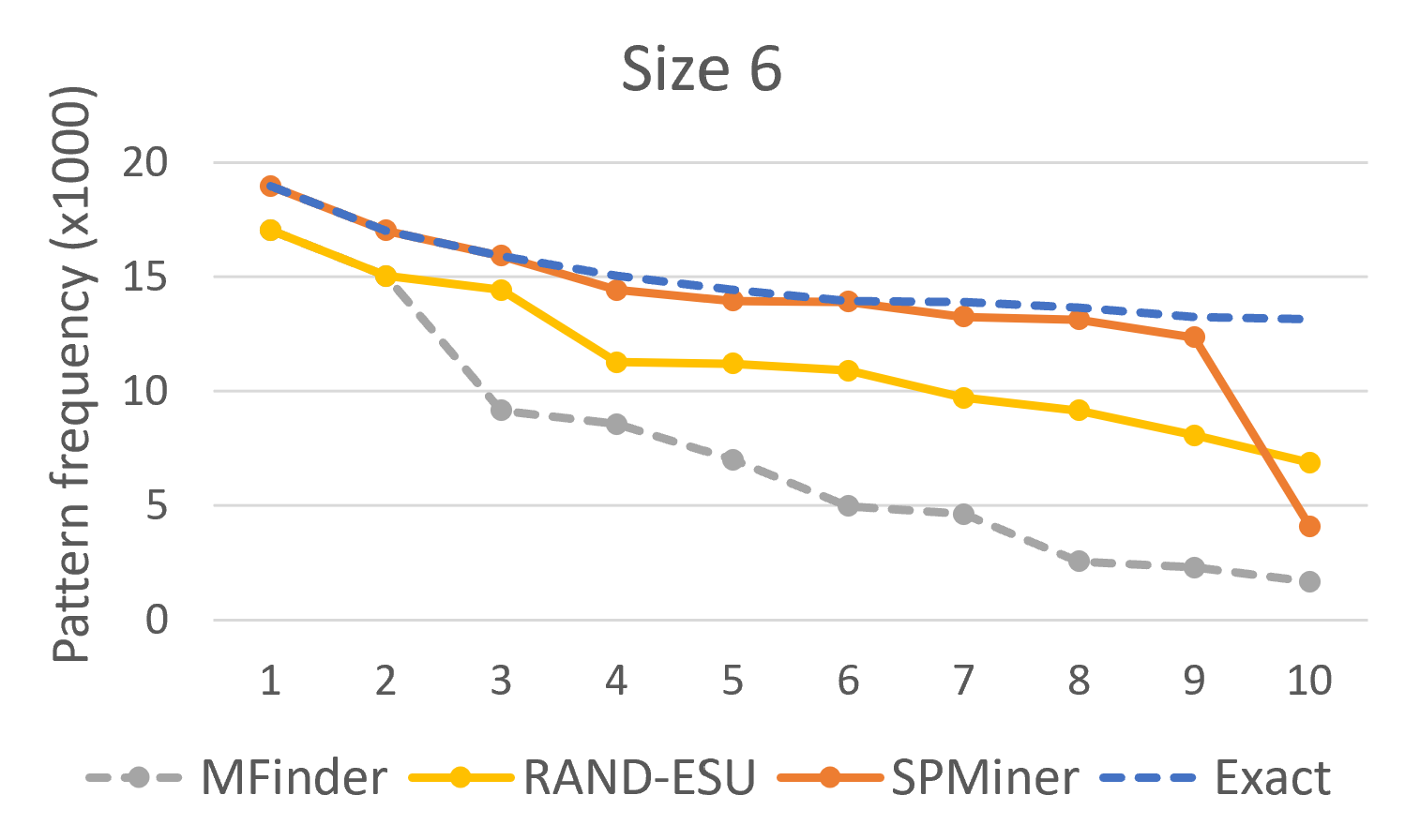}
    \end{tabular}
    \vspace{-5pt}
    \caption{
    \name vs. approximate frequent subgraph mining techniques: Among size-6 motifs, SPMiner is able to correctly identify the top $K$ most frequent motifs more accurately than baselines (left). Furthermore, the top 10 motifs identified by SPMiner have higher frequency than those found by baselines, for size 5 (middle) and size 6 (right) motifs. The blue dotted line represents the frequency of the groundtruth most frequent motifs.}
    \label{fig:small_motifs}
\end{figure*}

\section{Experiments}

To our knowledge, \name\ is \emph{the first method for learning neural models to perform frequent motif mining}. 
Here, we propose the first benchmark suites of experimental settings, baselines and datasets to evaluate the efficacy of neural frequent motif mining models.


\subsection{Experimental setup} 
We perform the following experiments:
(1) \textbf{Small motifs.} We experiment with small motifs of 5 and 6 nodes, where exact enumeration methods~\cite{hovcevar2014combinatorial,cordella2004sub} are able to find ground-truth most frequent motifs. We show that \name nearly perfectly finds these most frequent motifs. 
(2) \textbf{Large planted motifs.} To evaluate performance for mining larger motifs (computationally prohibitive for exact enumeration methods), we plant a large 10-node motif many times in a dataset (details in Appendix). Again, we show that \name is able to identify the planted motif as one of the most frequent in the dataset.
(3) \textbf{Large motifs in real-world datasets} We also compare \name against approximate methods~\cite{wernicke2006efficient,kashtan2004efficient} that scale to motifs with over 10 nodes. Here we find that \name identifies large motifs that are 100x more frequent than the ones identified by approximate methods in real-world datasets.
(4) \textbf{Runtime Comparison} 
We compare \name against non-neural exact methods, gSpan and Gaston~\cite{yan2002gspan,nijssen2005gaston}, as well as the highly accurate approximate method Motivo~\cite{bressan2021faster} to show that although more accurate, these methods are exponentially more expensive with respect to the size of subgraph patterns.
(5) \textbf{Encoder validation} validates the representations learned by the encoder architecture, and demonstrates the superior generalizability of the order embedding space through an ablation analysis.

\xhdr{Approximate baselines}
We compare against two widely-used approximate sampling-based motif mining algorithms, \textsc{Rand-ESU}~\cite{wernicke2006efficient} and \textsc{MFinder}~\cite{kashtan2004efficient}. For \textsc{MFinder}, we omit the slow $O(n^{n+1})$ exact probability correction. \textsc{MFinder} takes a degree-weighted sampling approach to seek motifs containing high-degree hub nodes. \textsc{Rand-ESU} iteratively expands candidate motifs one node at a time by maintaining a candidate set to ensure unbiased sampling. We tune hyperparameters of these baselines and \name\ so that they sample a  comparable number of subgraphs and achieve comparable wall clock runtime (details in Appendix \ref{appendix: runtime}).

\xhdr{Datasets}
We mine frequent motifs in a variety of domains, including biological (\textsc{Enzymes}), chemical (\textsc{Cox2}) and image (\textsc{MSRC}) datasets~\cite{borgwardt2005protein,sutherland2003spline,neumann2016propagation}. Table \ref{tab:real_dataset_stats} in Appendix \ref{appendix:implement detail} shows the graph statistics of the datasets used in our experiments. 
All datasets have been made public.
We focus on the topological structure of these datasets, omitting node labels when identifying frequent motifs; incorporation of labels is an interesting avenue for further work. 

\subsection{Results}

\xhdr{(1) Small motifs}
First, we experiment with small motifs of 5 and 6 nodes, where exact enumeration methods~\cite{hovcevar2014combinatorial,cordella2004sub} are able to find true most frequent motifs.
We pick an existing dataset, \textsc{Enzymes}, and count the exact motif counts for all possible motifs of size 5 and 6. 
We use the metric \emph{hit rate} at $k$, which measures the proportion of top-$k$ frequent motifs identified by \name and baselines, that are within the ground truth top-$k$ most frequent motifs by exact enumeration.
Figure \ref{fig:small_motifs} left shows that \name consistently achieves higher hit rate compared to baselines for mining size 6 motifs.\footnote{Note that as the rank increases, each method is required to identify a larger set of queries, thus the hit rate may increase or decrease.}

We further compare the frequencies of the top 10 motifs found by exact enumeration, \name and the baselines, \textsc{MFinder} and \textsc{Rand-ESU} for size 5 and 6 motifs. We observe that \name\ can consistently identify the size 5 and size 6 motifs whose frequencies are within 90\% of that of the groundtruth.
Furthermore, \name\ runs in 5 minutes, versus 10 hours for exact enumeration with (exact) \textsc{ESU} \cite{wernicke2006efficient} using the same hardware (see Appendix).

\xhdr{(2) Large planted motifs}
Currently exact methods for finding most frequent motifs are prohibitively expensive for larger motifs, and hence groundtruth frequency is hard to obtain. To evaluate identification of larger motifs, we randomly generate a large motif pattern $Q$ of size $n_1=10$, and randomly attach an instance of the motif to base graphs of size $n_2=10$ generated using a synthetic generator. 
Each motif is attached to base graph by randomly connecting a nodes from the motif with a node from the base graph. This process ensures that this pattern is one of the most frequent in the dataset. 
We repeat this process to generate a dataset of $1000$ graphs.

\begin{figure}[tb]
\centering
\includegraphics[width=0.36\textwidth]{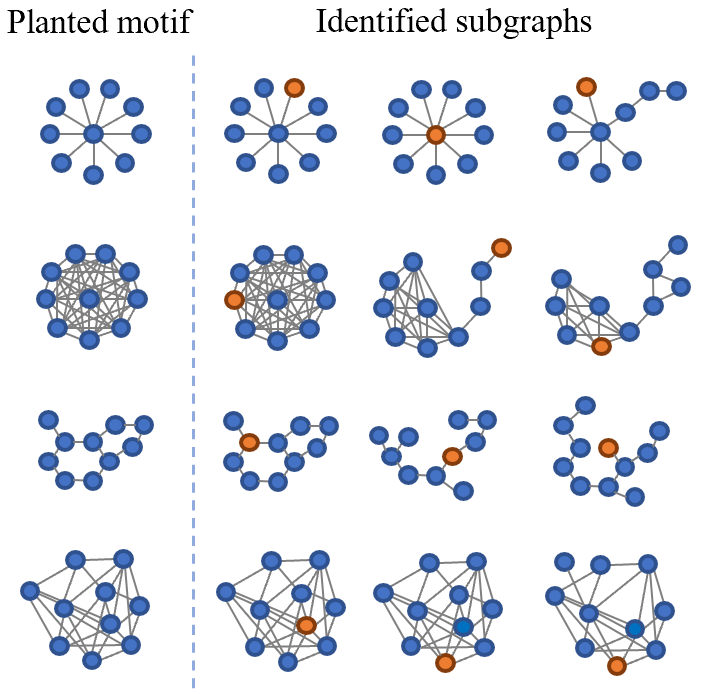}
\caption{The frequent subgraphs identified by \name (under Definition 1) closely match the planted motifs. Orange nodes denote the anchor node of the motif identified in \name.
}
\label{fig:planted-patterns}
\end{figure}

%

\begin{figure*}[tb]
    \centering    
    \includegraphics[width=0.9\textwidth]{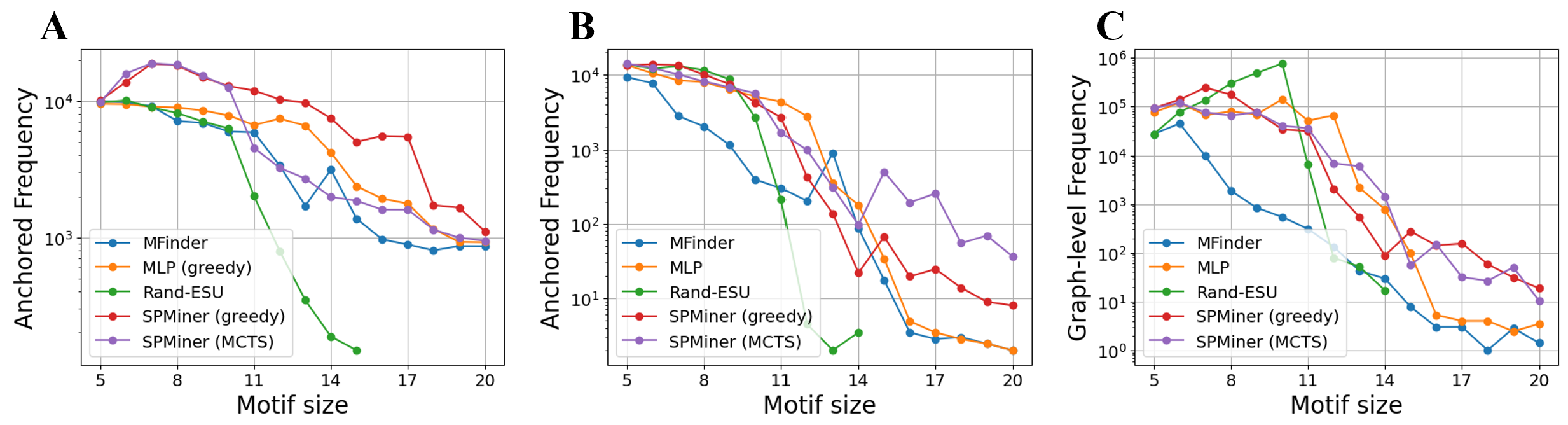}

    \caption{Comparison of median frequencies of motifs identified by different search strategies (\name-greedy and \name-MCTS) and baselines (MFinder, Rand-ESU); higher is better. 
    Across all motif sizes, \name finds patterns with higher node-anchored frequency than do the neural MLP baseline or the Rand-ESU and MFinder sampling-based baselines, across \textsc{COX2} (\textbf{A}), \textsc{Enzymes} (\textbf{B}), and \textsc{Enzymes} (\textbf{C}) datasets.
    }
    \label{fig:log-freqs}
\end{figure*}

Then, we use \name to identify frequent motifs. We expect \name to output patterns that resemble the planted motif. Figure \ref{fig:planted-patterns} shows three examples of the top 10 identified frequent patterns outputted by \name. We observe that for all examples, the planted motif is identified as one of the most frequent by \name\footnote{See Appendix for more details on all identified motifs.}.

\xhdr{(3) Identifying large motifs in real-world datasets}
Lastly, we compare against several baselines for finding large frequent motifs in graph data. Our baselines are \textbf{\textsc{MFinder}}, \textbf{\textsc{Rand-ESU}} and a neural baseline \textbf{\textsc{MLP}} that replaces order embeddings with an MLP and cross entropy loss. For varying motif sizes $k$, we take the top ten candidates for the most frequent motif for each method, and compute their true mean frequency via exact subgraph matching. Compared to approximate methods, \name is able to identify motifs that appear 10-100x more frequently, as seen in Figure \ref{fig:log-freqs}, particularly for graphs with size $>12$. The MCTS search variant of \name discovers frequent patterns of over 15 nodes in \textsc{Enzymes}, while no baseline can find motifs of median frequency more than $3$.
Note that even though \name's primary objective is to maximize anchored frequency (Definition \ref{def:motif_freq}), it outperforms all baselines on large motifs by 10-100x with the graph-level Definition \ref{def:motif_freq_graph} as well (Figure \ref{fig:log-freqs} right).
We further present the results for other datasets in Appendix \ref{appendix:implement detail}.



\xhdr{(4) Runtime Comparison}
We consider state-of-the-art exact motif mining methods, gSpan~\cite{yan2002gspan} and Gaston~\cite{nijssen2005gaston}, as well as the highly accurate approximate method Motivo~\cite{bressan2021faster}. These methods can produce exact, or highly accurate frequent subgraphs; at the cost of exponentially increasing cost with respect to subgraph pattern sizes.
Note that although an approximate method, Motivo use an expensive build-up phase to color the target before sampling~\cite{bressan2021faster}.
For each dataset, we adapt their code in C++, and tune the support threshold parameter~\cite{yan2002gspan} in order to obtain at least $10$ frequent motifs of the specified size, without exceeding a runtime budget of 1 hour and memory budget of 50GB.
Implementation details of these baselines are further explained in Appendix.

\begin{figure}[t]
    \centering   
    \includegraphics[width=0.47\textwidth]{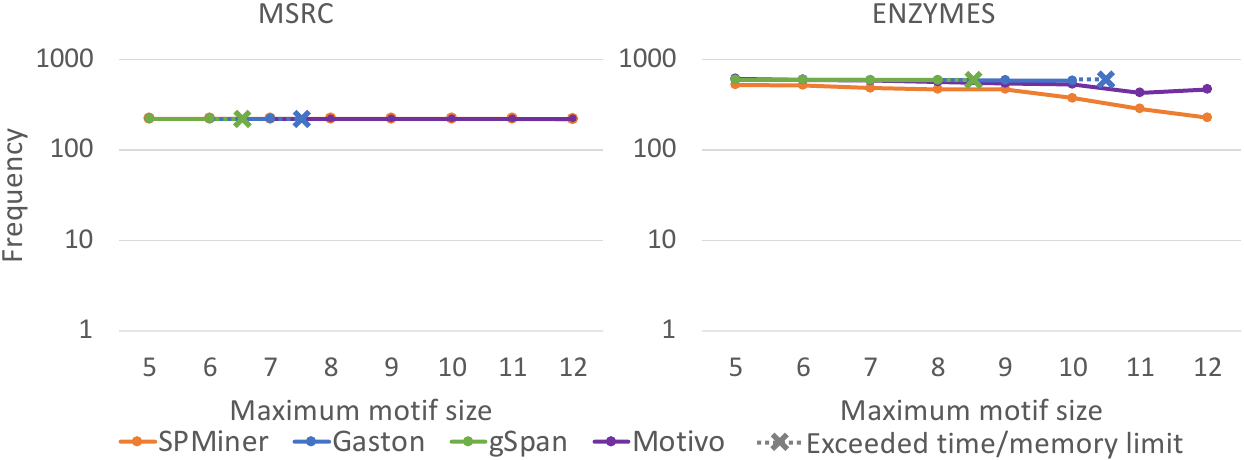}
    \caption{Frequency of motifs identified by Gaston, gSpan, Motivo, and \name. 
    \name is able to identify high-frequency motifs of large size. 
    }
    \label{fig:freq_vs_size}
\end{figure}

\begin{figure}[t]
    \centering    
    \includegraphics[width=0.47\textwidth]{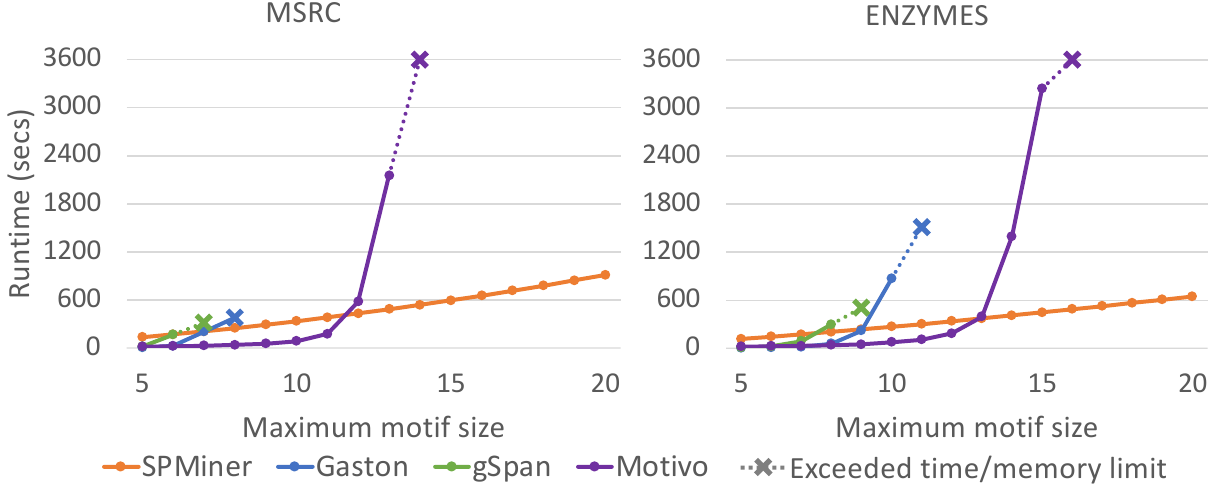}
    \caption{Runtime comparison between non-neural methods and \name. The curves for Gaston, gSpan, and Motivo end early due to exceeding memory or time limit for larger motifs.}
    \label{fig:runtime_vs_size}
\end{figure}

In Figure \ref{fig:freq_vs_size}, we plot the frequency of the motif identified by \name and the exact baseline methods, against the size of the motifs to be identified. 
\name identifies small frequent motifs (of size less than 10) whose frequency is at least $90\%$ that of the groundtruth most frequent motifs (by exact methods). 
For both datasets, gSpan and Gaston exceed the resource budget when identifying motifs of size larger than 10, while \name can still identify them efficiently.

We further compare the runtimes. Note that \name uses synthetic graph pretraining which takes fixed amount of time (8 hours) and shared for all experiments. ~\footnote{See Appendix \ref{appendix: runtime} for runtime details}.
Hence we report the runtime of \name at inference time. We show that the runtime of the exact methods grows exponentially, whereas \name has a linear trend as the size of motif grows.
In Figure \ref{fig:runtime_vs_size}, we plot the runtime required against the size of frequent motifs identified by \name and the exact methods. We observe that even with more computational budget, baseline methods quickly become intractable due to exponential increase in runtime.


\begin{figure}[t]
        \begin{subtable}{0.4\textwidth}
                \begin{tabular}{cc}
                    \begin{subfigure}[t]{0.6\textwidth}
                        \centering
        \includegraphics[width=\textwidth]{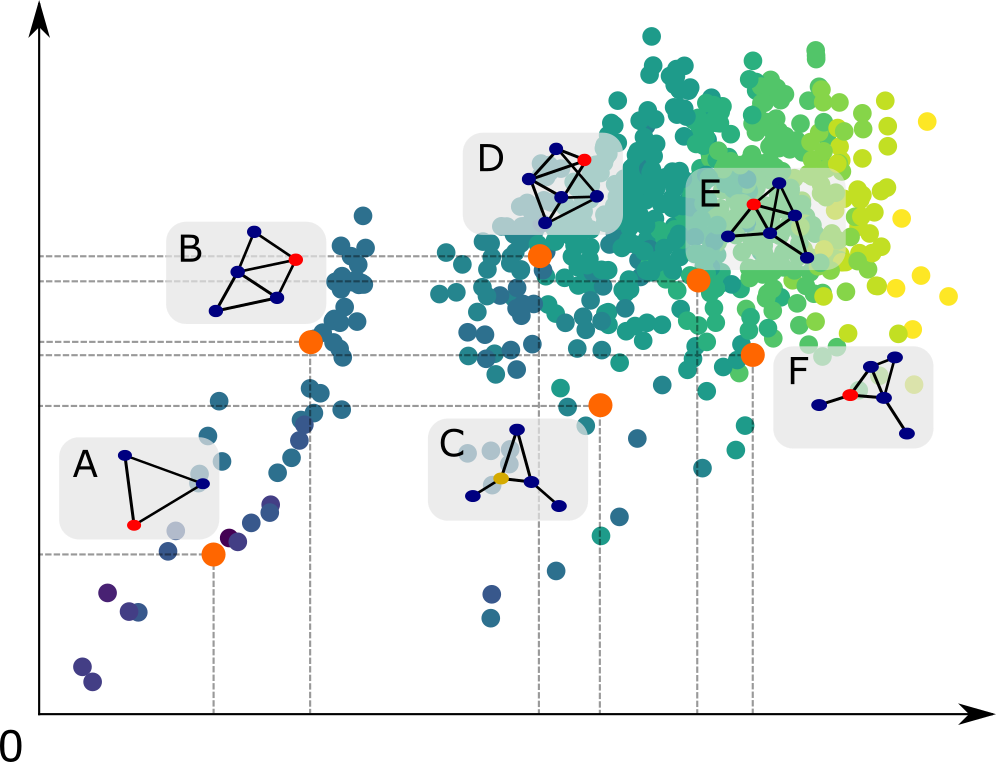}
                    \end{subfigure}
                &

            \begin{subfigure}[t]{0.4\textwidth}

            \includegraphics[width=\textwidth]{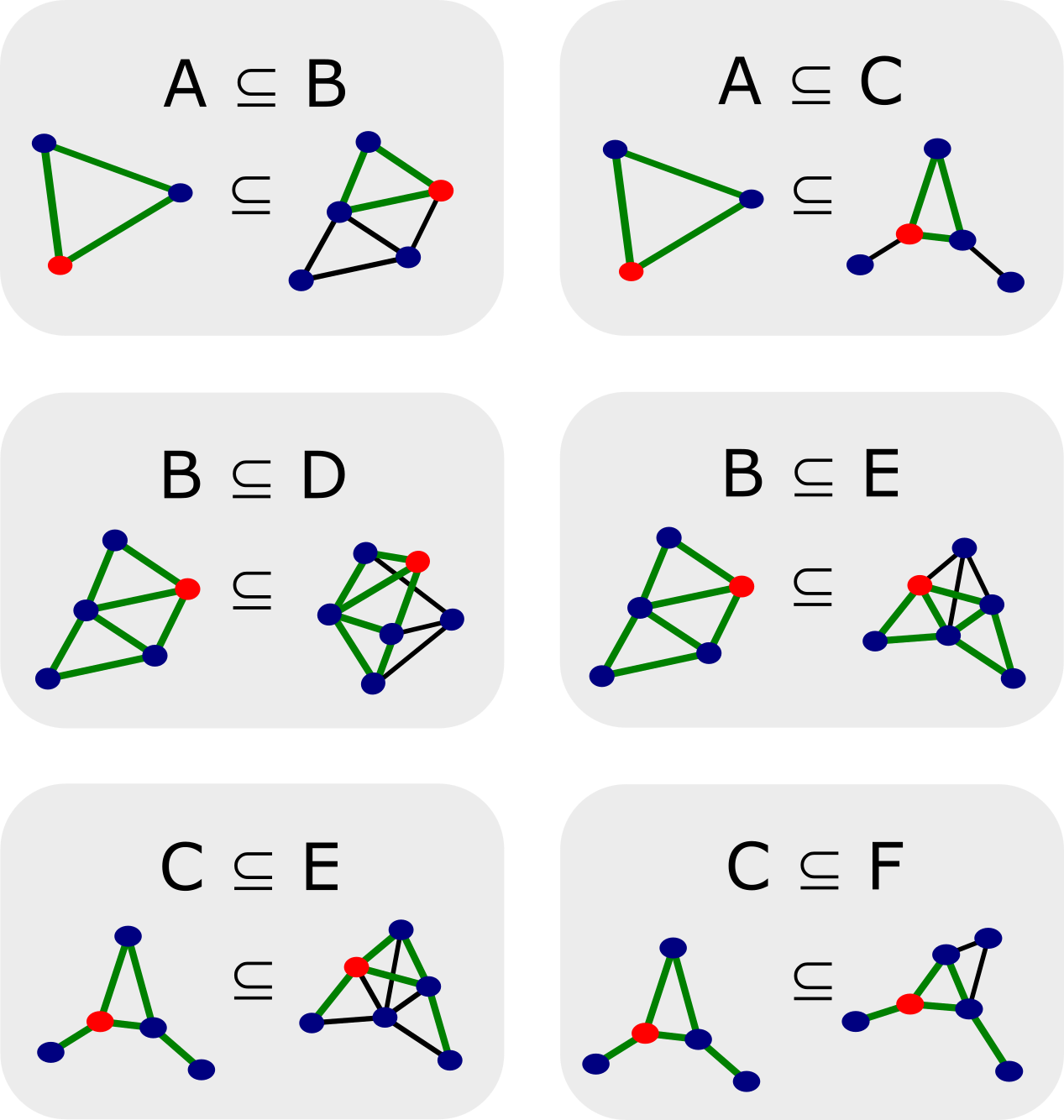}
            \end{subfigure}
            
            \end{tabular}
        \end{subtable}
    \caption{Two dimensions of the embedding space. Each point corresponds to a graph; color range from blue to yellow corresponds to small-to-large graphs.}
    \label{fig:encoder_embs}
\end{figure}


\xhdr{(5) Encoder validation}
Figure~\ref{fig:encoder_embs} demonstrates the structure of our order embedding space (here in just 2 dimensions). Notice how subgraphs are embedded to the lower left of their supergraphs. Experiments shows that \name achieves $95\%$ accuracy in determining whether one graph is a subgraph of the other. The learnable skip layer also contributes to the encoder performance gains by increasing the accuracy and area under PR curve. Please refer to appendix for subgraph relation performance of our GNN model component and the ablation studies.
\section{Limitations}
\name is the first approach to mine large frequent motifs in graph datasets.
However, there are still limitations to this pioneering approach.
The algorithm does not directly optimize for the graph-level frequency definition (Definition 2), although we observe that in practice the subgraphs identified by \name are still frequent (see Figure \ref{fig:log-freqs}(c)).
Additionally, \name only provides a list of frequent subgraph patterns via search, but not an accurate count prediction for a given subgraph pattern. Future work in approximating the \# P problem of subgraph counting is needed.
Finally, although in experiments, we find that distinguishing between the anchor node and other nodes in the neighborhood via node features results in more expressive GNNs beyond the WL test, further work on more expressive GNNs can further improve \name.
We hope that \name opens a new direction in graph representation learning and embedding space search to solve graph mining problems. 

\section{Conclusion}
We propose \name, the first neural framework for identifying frequent motifs in a target graph using graph representation learning. \name learns to encode subgraphs into an order embedding space, and uses a novel search procedure on the learned order embedding space to identify frequent motifs. \name advances the state-of-the-art, and is able to identify frequent motifs that are 10-100x more frequent than those identified by existing methods.

\clearpage

\bibliography{bibli}
\bibliographystyle{icml2020}

\clearpage

\appendix
\appendix

\section{Model Analysis}

\subsection{Proof of order embedding expressiveness}
\begin{proposition}
Given a graph dataset of size $n$, we can find a perfect order embedding of dimension $D$, where $D$ is the number of non-isomorphic node-anchored graphs of size no greater than $n$. The perfect order embeddings satisfy $z(G_Q) \preceq z(G)$ if and only if $G_Q$ is a subgraph of $G$.
\end{proposition}
\begin{proof}
\vspace{-3mm}
This perfect order embedding can be simply constructed by enumerating all possible non-isomorphic node-anchored graphs of size no greater than $n$, and place the count of the $i$-th graph into the $i$-th dimension of the order embedding.

If graph $G_Q$ is a subgraph of graph $G$, then by definition there exists an isomorphism $f$ between $G_Q$ and a subgraph of $G$. Therefore each motif in $G_Q$ can be mapped to the corresponding motif in $G$ by $f$. Hence the motif count of $G_Q$ for any motif is no less than the corresponding motif count of $G$.

Conversely, if graph $G_Q$ is not a subgraph of graph $G$, then the motif isomorphic to $G_Q$ has count at least $1$ for $G_Q$, and $0$ for graph $G$, violating the order constraint. 
\end{proof}

\subsection{Motif Counts Prediction}
To demonstrate the capability of GNNs in predicting motif counts, we randomly generate graphs using synthetic and real datasets \cite{KKMMN2016} and train the \name GNN to predict (log) counts of small motifs. We find that GNNs can accurately estimate these counts, with relative mean squared error of 12\% across all datasets. Together with Proposition 1, it confirms the capacity of \name to learn order embeddings that capture the subgraph relation.

\begin{table}[ht]	
\centering
\resizebox{0.9\linewidth}{!}{
\begin{tabular}{clccccc}\cmidrule[\heavyrulewidth]{2-7}
& \textbf{Dataset} &\multicolumn{1}{c}{\textsc{E-R}} & \multicolumn{1}{c}{\textsc{COX2}} & \multicolumn{1}{c}{\textsc{DD}} & \multicolumn{1}{c}{\textsc{\textsc{MSRC\_21}}} & \multicolumn{1}{c}{\textsc{FirstMMDB}}  
\\ 
\cmidrule{2-7}
& \textsc{MSE} & 8.4  & 7.1  &  8.9  &  10.5  & 12.7    \\ 
& \textsc{Rel Err} & 9.4\%  & 12.5\%  &  11.8\%  &  13.4\%  & 11.5\% \\ 
\cmidrule[\heavyrulewidth]{2-7}
\end{tabular}}
\caption{The MSE and relative MSE of log motif counts.}
\label{tab:all_match_AUROC}
\vspace{-8mm}
\end{table}

\section{Further Implementation Details}
\label{appendix:implement detail}
\xhdr{Real-world datasets}

Table \ref{tab:real_dataset_stats} shows the graph statistics of the real-world datasets used in our experiments. 

\begin{table}[ht]	
    \centering
    \resizebox{0.97\linewidth}{!}{
    \begin{tabular}{clccccc}\cmidrule[\heavyrulewidth]{2-5}
    & \textbf{Dataset} & Number of graphs & Number of nodes & Number of edges
    \\ 
    \cmidrule{2-5}
    & \textsc{ENZYMES} & 600 & 19.6K & 37.3K \\
    & \textsc{COX2} & 467 & 19.3K & 20.3K \\
    & \textsc{MSRC\_9} & 221 & 9.0K & 21.6K \\
    & \textsc{MNRoads} & 1 & 2.6K & 3.3K \\
    & \textsc{COIL-DEL} & 3900 & 84.0K & 211.5K \\
    
    \cmidrule{2-5}
    
    & \textsc{Plant-Star} & 1000 & 20K & 30.2K \\
    & \textsc{Plant-Clique} & 1000 & 20K & 66.2K \\
    & \textsc{Plant-Molecule} & 1000 & 20K & 32.2K \\
    & \textsc{Plant-Random} & 1000 & 20K & 49.3K \\
    
    \cmidrule[\heavyrulewidth]{2-5}
    \end{tabular}}
    \caption{Graph statistics of datasets used in experiments. The \textsc{Plant} datasets are those used in Experiment (2).}
    \label{tab:real_dataset_stats}
    \vspace{-8mm}
\end{table}
    
\xhdr{Synthetic dataset}
\label{appendix: synthetic dataset}
We design a synthetic graph generator to provide training graphs for synthetic pretraining, so that the model learns an order embedding space over diverse types of graphs and can generalize to new datasets at inference time.

The generator randomly chooses one of the following ways to generate a graph:
Erd\H{o}s-R\'enyi (E-R)~\cite{erdds1959random}, Extended Barabási–Albert graphs~\cite{albert2000topology}, Power Law Cluster graphs~\cite{holme2002growing} and Watts-Strogatz graph~\cite{watts1998collective}. 
For the Erd\H{o}s-R\'enyi generator, the probability of an edge is $p \sim \textrm{Beta}(1.3, 1.3n/\log_2(n) - 1.3)$. For the Extended Barabási–Albert generator, the number of attachment edges per node is $m \sim \textrm{Unif}(1, 2\log_2 n)$, the probability of adding edges is $p \sim \textrm{Exp}(20)$ (capped at $0.2$) and the probability of rewiring edges is $q \sim \textrm{Exp}(20)$ (capped at $0.2$). For the Power Law Cluster generator, the number of attachment edges per node is $m \sim \textrm{Unif}(1, 2\log_2 n)$ and the probability of adding a triangle after each edge is $p \sim \textrm{Unif}(0, 0.5)$. For the Watts-Strogatz generator, each node is connected to $k \sim n\textrm{Beta}(1.3, 1.3n/\log_2(n) - 1.3)$ neighbors (minimum $2$) and the rewiring probability is $p \sim \textrm{Beta}(2, 2)$. Here $n$ is the desired graph size.

\xhdr{Training details} 
Our encoder model architecture consists of a single hidden-layer feedforward module, followed by 8 graph convolutions with ReLU activation and learned dense skip connections, followed by a 4-layer perceptron with 64-dimensional output. We use SAGE graph convolutions \cite{hamilton2017inductive} with sum aggregation and no neighbor sampling. We train the network with the Adam optimizer using learning rate $10^{-4}$ for 1 million batches of synthetic data, with batch size 64 and balanced class distribution.

\begin{table*}[t]
    \centering
\begin{tabular}{cc}
    \begin{subtable}[t!]{0.45\textwidth}
        \centering
        \resizebox{\linewidth}{!}{
        \begin{tabular}[t!]{clcccc}\cmidrule[\heavyrulewidth]{2-6}
        & \textbf{Dataset} &\multicolumn{1}{c}{\textsc{Enzymes}} & \multicolumn{1}{c}{\textsc{COX2}} & \multicolumn{1}{c}{\textsc{COIL-DEL}} & \multicolumn{1}{c}{\textsc{MNRoads}} 
        \\ 
        \cmidrule{2-6}
                    & \textsc{Rand-ESU} & 13:59 & 13:56 & 19:22 & 4:47 \\
                    & \textsc{MFinder} & 17:29 & 17:13 & 20:57 & 16:24 \\
                    & \name & 9:26 & 8:13 & 11:15 & 9:41 \\
        \cmidrule[\heavyrulewidth]{2-6}
        \end{tabular}}
        \caption{Runtime comparison with approximate methods}
    \end{subtable}
&
    \begin{subtable}[t!]{0.45\textwidth}
        \centering
        \resizebox{\linewidth}{!}{
        \begin{tabular}[t!]{clcccc}\cmidrule[\heavyrulewidth]{2-6}
        & \textbf{Dataset} &\multicolumn{1}{c}{\textsc{Enzymes}} & \multicolumn{1}{c}{\textsc{COX2}} & \multicolumn{1}{c}{\textsc{COIL-DEL}} & \multicolumn{1}{c}{\textsc{MNRoads}} 
        \\ 
        \cmidrule{2-6}
                    & \textsc{Rand-ESU} & 39235 & 29368 & 85149 & 7605 \\
                    & \textsc{MFinder} & 10000 & 10000 & 10000 & 10000 \\
                    & \name & 1500 & 1500 & 1500 & 1500 \\
        \cmidrule[\heavyrulewidth]{2-6}
        \end{tabular}}
        \caption{Average number of subgraphs sampled}
    \end{subtable}
    \end{tabular}
    \vspace{-2mm}
    \caption{Comparison of runtimes and number of subgraphs sampled (averaged over all motif sizes) for methods in Experiment \textbf{(3)}. \textsc{SPMiner} has the lowest runtimes and requires fewer samples to identify frequent motifs. We compare against the greedy variant of \name; the MCTS variant samples the same number of subgraphs and incurs small computational overhead.}
    \label{tab:runtime_comparison}
    \vspace{-7mm}
\end{table*}

\xhdr{Decoder configuration}
To sample node-anchored neighborhoods, we follow the iterative procedure of \textsc{MFinder}~\cite{kashtan2004efficient}: after picking a random node as the anchor, we maintain a search tree, in each step randomly picking a node from the frontier with probability weighted by its number of edges with nodes in the search tree. The procedure terminates when the search tree reaches $N$ nodes, and we take the subgraph induced by these nodes as the neighborhood. We select $N$ uniformly randomly from 20 to 29 for each neighborhood (except for Experiment \textbf{(2)}, where the maximum size is 25), sampling 10000 neighborhoods in total.

For a given maximum motif size $k$, \name generates all motifs up to size $k$ through a single run of the search procedure. For the greedy procedure, we sample 1000 seed nodes and expand each corresponding candidate motif up to size $k$, recording all candidates for intermediate sizes. For the MCTS procedure, we run a total of 1000 simulations, divided equally among each motif size up to $k$, reusing the existing search tree with each new iteration; this procedure creates a useful prior for each successive motif size. We use exploration constant $c=0.7$ throughout.

\xhdr{Baseline configuration}
For \textsc{MFinder} \cite{kashtan2004efficient}, we sample 10000 neighborhoods per motif size. We omit the slow $O(n^{n+1})$ exact probability correction, simply weighting each sampled neighborhood equally. We take the first sampled node as the anchor node in node-anchored experiments.

For \textsc{Rand-ESU} \cite{wernicke2006efficient}, we set the expansion probabilities for each search tree level $i$ according to $p_i = (1-i/(k+1))^\tau$ where $k$ is the maximum motif size and $\tau$ can be tuned for runtime depending on the dataset; we use $\tau=5$ for \textsc{Enzymes} (except $\tau=2.3$ in Experiment \textbf{(1)}), $\tau=2$ for \textsc{Cox2}, $\tau=2$ for \textsc{MNRoads} and $\tau=9$ for \textsc{Coil-Del}. We use the suggested variance reduction technique of sampling a fixed proportion of children to expand at each level.

For \textsc{Motivo} \cite{bressan2021faster}, we use the official cpp implementation with $10^7$ samples.


\section{Further Experimental Results}

\xhdr{Encoder Validation Details}

Table \ref{tab:encoder&ablation} evaluates the faithfulness of the order embedding space: given a pair of graphs, the task is to predict whether one is a subgraph of the other. We use accuracy and area under PR curve to evaluate the performance. \name achieves $95\%$ accuracy in determining whether one graph is a subgraph of the other.

\begin{table}[ht]
    \centering
    \vspace{-3mm}
    \resizebox{\linewidth}{!}{
        \begin{tabular}{ccccccccc}\cmidrule[\heavyrulewidth]{1-9}
            & \multicolumn{2}{c}{Synthetic} & \multicolumn{2}{c}{\textsc{Enzymes}} & \multicolumn{2}{c}{\textsc{COX2}} & \multicolumn{2}{c}{\textsc{FirstMMDB}} \\
            & Acc & AP & Acc & AP & Acc & AP & Acc & AP\\
            \cmidrule{1-9}
            MLP (SAGE) & \textbf{96.3} & 43.9 & 96.2 & 40.2 & 91.3 & 45.8 & 94.7 & 38.8 \\  
            MLP (GIN) & 96.1 & 41.0 & 95.4 & 34.8 & 89.0 & 37.3 & 94.7 & 35.9 \\  
            MLP (GCN) & 95.8 & 37.2 & 94.7 & 31.5 & 91.6 & 47.2 & 94.3 & 34.1  \\  
            Order (No skip) & 95.9 & 45.7 & 96.5 & 57.3 & 92.1 & 70.5 & 95.1 & 53.4 \\  
            \cmidrule{1-9}
            Order (Full) & 96.2 & \textbf{46.8} & \textbf{96.6} & \textbf{60.9} & \textbf{92.6} & \textbf{71.2} & \textbf{95.6} & \textbf{59.9} \\
        \cmidrule[\heavyrulewidth]{1-9}
    \end{tabular}
    }
    \caption{Accuracy and area under PR curve performance with different GNN architectures}
    \vspace{-8mm}
    \label{tab:encoder&ablation}
\end{table}
We further conduct ablation studies. A random model that outputs labels according to the class distribution would receive 3.5 AUPR in our task. 
The following architectures are considered:
(1) GCN+MLP: uses MLP with cross entropy loss to replace order embedding; uses the GCN \cite{kipf2016semi} architecture;
(2) GIN+MLP: same as (1) but with GIN \cite{xu2018representation} architecture;
(3) SAGE+MLP: same as (1) but with SAGE \cite{hamilton2017inductive} architecture\footnote{We use the SAGE architecture with sum aggregation, but without neighbor sampling.};
(4) No skip: same as \name (order embedding loss and the SAGE \cite{hamilton2017inductive} architecture), but does not use the proposed learnable skip layer.

Results in Table \ref{tab:encoder&ablation} demonstrate that both order embedding and the learnable skip layer are crucial to performance gains, together reaching over 60 AUPR on real-world datasets.


\xhdr{Frequency comparison of identified motifs}
We additionally run Experiment \textbf{(3)} on three more datasets, \textsc{Coil-Del}~\cite{riesen2008iam}, a dataset of 3D objects derived from COIL-100, a graph of the Minnesota road network ~\cite{netrepo}, abbreviated \textsc{MNRoads} and the Arxiv dataset \cite{hu2020ogb} with 100K nodes and 1M edges to demonstrate that \name can efficiently process larger graphs. We find again that \name\ consistently identifies more frequent motifs than the baselines, especially for large motifs (Figure \ref{fig:extra_freq_comparison}).
\begin{figure}
    \centering
    \includegraphics[width=0.18\textwidth]{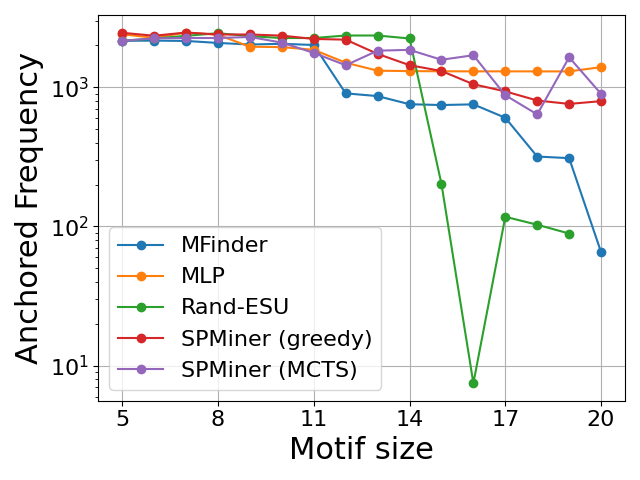}
    \hspace{1cm}
    \includegraphics[width=0.18\textwidth]{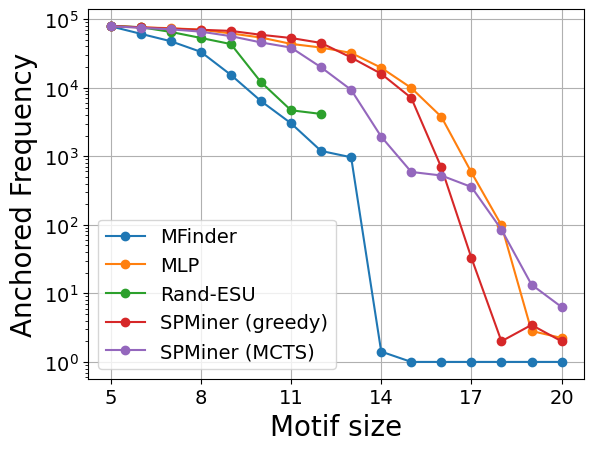}
    \vspace{-4mm}
    \caption{Median frequencies of motifs identified by \name and baselines for each graph size; higher is better. Top: \textsc{MNRoads}; bottom: \textsc{Coil-Del}.}
    \label{fig:extra_freq_comparison}
    \vspace{-2mm}
\end{figure}

Table \ref{tab:arxiv_freqs} shows the node-anchored frequency of motifs identified by \name and MFinder on the larger Arxiv dataset. Due to the computational expense of getting ground-truth frequencies for the identified motifs, we estimate these frequencies by anchoring the target graph at randomly sampled anchor points, and only checking subgraph isomorphism between the query and target anchored at the sampled points. \name identifies more frequent motifs in the regime of larger motif sizes.

\begin{table}[ht]
    \centering
    \resizebox{0.60\linewidth}{!}{
    \begin{tabular}{c|c|c|c}
        \toprule
        \textbf{Method} & Size 8 & Size 9 & Size 10 \\
        \hline
        \name & 155.1K & \textbf{153.8K} & \textbf{151.0K}  \\
        MFinder & \textbf{156.4K} & 152.3K & 146.3K \\
        \bottomrule
    \end{tabular}
    }
    \caption{Median frequencies of motifs identified by \name and baselines for each graph size.}
    \label{tab:arxiv_freqs}
    \vspace{-7mm}
\end{table}

\section{Runtime Comparison}
\label{appendix: runtime}
\name uses a pre-training stage to train the encoder, and a search stage to identify frequent subgraphs (inference stage).
The pre-training stage takes 8 hours, which is only performed once for all datasets in the experiment section (see dataset statistics). We further perform runtime comparison at inference stage, compared to other baselines.
\xhdr{Runtime comparison with approximate methods}
We tune the hyperparameters so that all methods are comparable in runtime and sample a comparable number of subgraphs. All methods are single-process and implemented with Python; the neural methods are implemented with PyTorch Geometric~\cite{Fey/Lenssen/2019}. We run all methods on a single Xeon Gold 6148 core; the neural methods additionally use a single Nvidia 2080 Ti RTX GPU. Table \ref{tab:runtime_comparison} shows that \name\ runs in less time and requires fewer subgraph samples than the baselines.

\xhdr{Runtime comparison with exact methods}
In experiments, we compare with exact methods gSpan~\cite{yan2002gspan}, Gaston~\cite{nijssen2005gaston} and Motivo~\cite{bressan2021faster}. For gSpan, we use a highly optimized C++ implementation\footnote{\url{https://github.com/Jokeren/gBolt}}, which we modify to prune the search once the identified motif reaches the specified maximum motif size, in order to increase its efficiency in our setting. We use the single-threaded version. 
For Gaston, we use the official C++ implementation\footnote{\url{http://liacs.leidenuniv.nl/~nijssensgr/gaston/download.html}}, using the variant with occurrence lists and specifying the desired maximum motif size through a command-line argument. 
For Motivo, we use the official C++ implementation\footnote{\url{https://gitlab.com/steven3k/motivo/}}, using 4 threads and $10^7$ samples through command-line arguments.
We impose a resource limit of two hours of runtime and 50GB of memory for both methods. We run both methods on a single Xeon Gold 6148 core.

\end{document}